\documentclass[11pt]{article}

\usepackage[T1]{fontenc}
\usepackage{times}
\usepackage{graphics}
\usepackage[utf8x]{inputenc}
\usepackage{verbatim}
\usepackage{amsfonts}
\usepackage{amsmath}
\usepackage{amssymb}
\usepackage{listings}
\usepackage{latexsym}
\usepackage{graphicx}
\usepackage{epsfig}
\usepackage{cite}
\usepackage{url}
\usepackage{algorithmic}
\usepackage{eso-pic}
\usepackage{fancyvrb,relsize}
\usepackage{fancyhdr}
\usepackage{multicol}

\renewcommand{\baselinestretch}{1.0}

\newcommand{\sk}{\smallskip\noindent}

\newcommand{\vect}[1]{\mathbf{#1}}
\newcommand{\avg}[1]{\left \langle #1 \right \rangle}
\DeclareMathOperator*{\argmin}{arg\,min}

\newtheorem{theorem}{Theorem}
\newenvironment{proof}
        {{\bf Proof.}}
        {\hspace*{\fill}$\Box$\par\vspace{4mm}}

\newcounter{labenumi}
\newenvironment{labeled-enumerate-i}[2]
        {\begin{list}
         {#1#2\arabic{labenumi}.}
         {\usecounter{labenumi}
          \setlength{\labelwidth}{1cm}
          }
         }
        {\end{list}}
\newcounter{labenumii}
\newenvironment{labeled-enumerate-ii}[4]
        {\begin{list}
         {#1#2#3#4\arabic{labenumii}.}
         {\usecounter{labenumii}
          \setlength{\labelwidth}{1cm}
          }
         }
        {\end{list}}

\begin{document}

%%%%%%%%%%%%%%%%%%%%%%%%%%%%%%%%%%%%%%%%%%%%%%%%%%%%%%%%%%%%%%%%%%%%%%%%%%
\renewcommand{\baselinestretch}{1.0}
\normalsize

%%%%%%%%%%%%%%%%%%%%%%%%%%%%%%%%%%%%%%%%%%%%%%%%%%%%%%%%%%%%%%%%%%%%%%%%%%
\vspace{-2.0cm}
\title{
On a model for integrated information\footnote{Contact Author: Enrico Nardelli, \textsf{nardelli@mat.uniroma2.it}}
}

%%%%%%%%%%%%%%%%%%%%%%%%%%%%%%%%%%%%%%%%%%%%%%%%%%%%%%%%%%%%%%%%%%%%%%%%%%
\author{
{Alessandro Epasto$^{1}$} \and {Enrico Nardelli$^{1}$} \and
\hspace*{-1cm}
\begin{minipage}{0.95\textwidth}
\vspace{-0.5cm}
\small
\bigskip
\newcounter{affil}
\begin{list}
{\arabic{affil}.}
{\usecounter{affil}
 \setlength{\labelwidth}{0.5cm}
  \setlength{\rightmargin}{\leftmargin}
 }
\item Dept. of Mathematics, Univ. of Roma Tor Vergata, Roma, Italy.
\end{list}
\end{minipage}
}

%%%%%%%%%%%%%%%%%%%%%%%%%%%%%%%%%%%%%%%%%%%%%%%%%%%%%%%%%%%%%%%%%%%%%%%%%%
\date{\small \it Printed: \today}
\pagestyle{myheadings}
\markright{\small \it Version 1.2.1 --- Last Revision: 30 December 2009}

%%%%%%%%%%%%%%%%%%%%%%%%%%%%%%%%%%%%%%%%%%%%%%%%%%%%%%%%%%%%%%%%%%%%%%%%%%
\maketitle

%%%%%%%%%%%%%%%%%%%%%%%%%%%%%%%%%%%%%%%%%%%%%%%%%%%%%%%%%%%%%%%%%%%%%%%%%%
\vspace{-0.5cm}
\begin{abstract}
\renewcommand{\baselinestretch}{1.0}
\normalsize
\noindent
In this paper we give a thorough presentation of a model proposed by Tononi et al. for modeling \emph{integrated information}, i.e. how much information is generated in a system transitioning from one state to the next one by the causal interaction of its parts and \emph{above and beyond} the information given by the sum of its parts. We also provides a more general formulation of such a model, independent from the time chosen for the analysis and from the uniformity of the probability distribution at the initial time instant. Finally, we prove that integrated information is null for disconnected systems.
\end{abstract}
\vspace{-0.5cm}

\sk {\bf Keywords}:
\renewcommand{\baselinestretch}{1.0}
\normalsize
integrated information, effective information, information theory, neural networks, probabilistic boolean networks.
%\vspace{+1cm}

%%%%%%%%%%%%%%%%%%%%%%%%%%%%%%%%%%%%%%%%%%%%%%%%%%%%%%%%%%%%%%%%%%%%%%%%%%
%%%%%%%%%%%%%%%%%%%%%%%%%%%%%%%%%%%%%%%%%%%%%%%%%%%%%%%%%%%%%%%%%%%%%%%%%%
\section{Introduction}
%%%%%%%%%%%%%%%%%%%%%%%%%%%%%%%%%%%%%%%%%%%%%%%%%%%%%%%%%%%%%%%%%%%%%%%%%%
\label{intro}
\renewcommand{\baselinestretch}{1.3}
\normalsize
The term \emph{integrated information} (denoted $\phi$, for short) has been introduced by Giulio Tononi~\cite{Tono-01-AIB, ToSp-03-BMC, Tono-04-BMC} to characterize the capacity of a system to integrate information acquired by its parts. Informally speaking, the integrated information owned by a system in a given state can be described as the information (in the Theory of Information sense) generated by a system in the transition from one given state to the next one as a consequence of the causal interaction of its parts above and beyond the sum of information generated independently by each of its parts.

Such a theory was first introduced as a linear model~\cite{ToSE-94-PNAS, ToSE-96-PNAS, ToES-98-TCS, ToSp-03-BMC, Tono-04-BMC}, then reformulated as a discrete one~\cite{BaTo-08-PCB, Tono-08-BB, BaTo-09-PCB} and was aimed at trying to formally capture what is consciousness in living beings~\cite{ToEd-98-Sci, EdTo-00, Tono-08-BB}. Its description is not always clear from a mathematical point of view, and to best of our knowledge this is the first formal description where all steps of the model are presented in detail using the framework of probabilistic boolean networks.

In our presentation we also provides a more general formulation of the model, which can be used for analyzing the system at a generic time instant, and which does not require the assumption of uniformity of the probability distribution at the initial time instant.

We also formally prove here, for the first time in the literature to the best of our knowledge, that integrated information is null for a disconnected system, that is a system made up by independent components.

\sk
The characterization of integrated information is based on another concept, always defined by Tononi and coauthors, named \emph{effective information} and modeling how much information is gained by an external observer on the previous state of a system from checking which is its current state, with respect to what can "a priori" be deduced on the previous state from the known dynamics of the system itself. Given this emphasis on the experimental side of the knowledge acquisition process, we suggest here to use the terms "experimental information" or "Galileian information" as synonyms for "effective information".

Effective information is zero for static systems or uniformly random systems, which is consistent with everyday scientist's experience. And, similarly, integrated information is also zero for disconnected systems, independently from their kind.

%%%%%%%%%%%%%%%%%%%%%%%%%%%%%%%%%%%%%%%%%%%%%%%%%%%%%%%%%%%%%%%%%%%%%%%%%%
%%%%%%%%%%%%%%%%%%%%%%%%%%%%%%%%%%%%%%%%%%%%%%%%%%%%%%%%%%%%%%%%%%%%%%%%%%
\section{Probabilistic Boolean Networks}
%%%%%%%%%%%%%%%%%%%%%%%%%%%%%%%%%%%%%%%%%%%%%%%%%%%%%%%%%%%%%%%%%%%%%%%%%%
\label{probabilistic-boolean-networks}
Let $X=(V,E)$ be a directed graph with $n$ boolean nodes, i.e. taking values in $\{0,1\}$. The value taken by a node is called also its \emph{state}. Edge $(u,v) \in E$ models the fact that node $v$ gets in input the state of $u$. We assume time runs in discrete steps or instants, and nodes may change their value with the flow of time depending on (the value of) the states of their input nodes.

Temporal evolution of state of node $i$ is given by a law $f_i:\{0,1\}^{n_i} \rightarrow \{0,1\}$ computing state of $i$ at the next time instant as a function only of the current state of its $n_i\leq n$ input nodes. Self loops are admitted. Nodes can all have the same law $f$ or each node can have its specific law. In any case laws are constant with time.

We call $X$ as defined above a \emph{Deterministic Boolean Network}. To put things into context, \emph{Random Boolean Networks} have been defined in the literature since many years, differing from the deterministic version only in the fact that each $f_i$ is randomly chosen when building the network. Random boolean networks have been widely studied as model for gene expression in biological systems.

Various probabilistic versions of Boolean Networks have also been defined, different from ours, for example~\cite{SDKZ-02-BI}, where each node at each time instant randomly chooses, according to a given probability distribution, the law to be used from a finite domain of admissible laws.

\sk
Our version of \emph{Probabilistic Boolean Network} (PBN, for short) assumes the probabilistic law $r_i:\{0,1\}^{n_i} \rightarrow [0,1]$ associated to node $i$ provides for each configuration of the states of the $n_i$ input nodes the probability $r_i$ that at the next time instant node $i$ has (equivalently, is in) state $1$ (being then $1-r_i$ the probability $i$ is in state $0$). It can be shown that this model can describe every network defined according to the model introduced in~\cite{SDKZ-02-BI}. In the following we use interchangeably the terms system and network.

\sk
At each time instant $t$ a PBN can be in any of its $2^n$ states, we assume are provided of some arbitrary enumeration $\{x_i\}$. State of network $X$ at time $t$ is denoted $X_t$. A PBN can also be considered as a \emph{Markov chain with a finite space state}.

\sk
A PBN is completely described by its \emph{state transition matrix} $S$, whose elements $s_{i j}$ are:
\[
  s_{i j} \circeq p(X_{t+1}=x_j \: | \: X_t=x_i) \nonumber
\]
that is, element $s_{i j}$ is the probability that at time $t+1$ the network is in state $x_j$ {\em conditioned} to the fact that at time $t$ the network was in state $x_i$. Note that since the probabilistic law associated to each node is time constant, state transition matrix $S$ is also time constant, hence we can speak of an \emph{homogeneous Markov chain}. A square matrix of real numbers is a state transition matrix if $0\leq s_{i j}\leq 1$ e $\sum_{i=1}^{n} s_{i j} = 1$.

\sk
Values of $s_{i j}$ can be easily computed by means of the $r_k$ values for each node $k$ as it follows. Let $i=\sigma_n \sigma_{n-1} \ldots \sigma_1$ be the bit string representing the network state at instant $t$, where $\sigma_k$ represent state of node $k$ at instant $t$. The network state at the next instant $t+1$ is $j=\sigma'_n \sigma'_{n-1} \ldots \sigma'_1$ where $\sigma'_k$ is the state of node $k$ computed by law $r_k$ for instant $t+1$. It is $\sigma'_k=1$ with probability $r_k(\sigma_n \sigma_{n-1} \ldots \sigma_1)$ and $\sigma'_k=0$ with probability $1-r_k(\sigma_n \sigma_{n-1} \ldots \sigma_1)$. Then
\begin{equation}
\label{s-from-r}
s_{i j}=\prod_{k=1}^{n} \rho_k \nonumber
\end{equation}
where $\rho_k=r_k(\sigma_n \sigma_{n-1} \ldots \sigma_1)$ if $\sigma'_k=1$ and $\rho_k=1-r_k(\sigma_n \sigma_{n-1} \ldots \sigma_1)$ if $\sigma'_k=0$.

\sk
Let us denote with $\vect{p}_t(i)=p(X_t=x_i)$ probability that network is in state $x_i$ at instant $t$. State distribution probability at $t+1$ is given by:
\[
 \vect{p}_{t+1}(x_i) = \sum_{j=1}^{2^n} {\vect{p}_t(x_j) s_{j i} }
\]
Note that, even if $S$ is time constant (i.e., stationary), state probability distribution is not necessarily so. Let $\vect{p}_t$ be the row vector with elements $\vect{p}_t(i)$. Previous formula can be written in a matrix form as
\[
\vect{p}_{t+1} = \vect{p}_t \cdot S
\]
and, denoting with $S^i$ the $i$-th column of $S$, it is
\[
\vect{p}_{t+1}(i) = \vect{p}_t \cdot S^i
\]
If for some $t$ it is  $\vect{p}_{t+1}(\cdot)= \vect{p}_t(\cdot)$ then we say the network is in the \emph{stationary regime}. It is then
\[
\vect{p} = \vect{p} \cdot S
\]
that is $\vect{p}$ is an eigenvector of $S$ with eigenvalue $1$. Note that not every eigenvector of $S$ can be a stationary probability distribution, since it has to fulfill probability distribution constraints. For example, the null eigenvector is never a stationary probability distribution.

Row $S_i$ of the state transition matrix provides the conditional probability distribution $p(X_{t+1}\,|\,X_{t}=x_i)$ describing network state at the instant \emph{next} to the one the network is in state $x_i$.

\sk
Network dynamics can also be analyzed backwards in time. Let us assume that we have observed or measured that network at instant $t$ is in a given state. We can then compute state distribution probability for instant $t-1$, that is we can compute the law by which states at instant $t-1$ might have caused the state actually observed or measured at instant $t$. This is provided by defining a \emph{state backward-transition matrix} $B$, describing probabilities obtained inverting through Bayes rule the relations between events. Its elements $b_{i j}$ are:
\[
  b_{i j}(t) \circeq p(X_{t-1}=x_j \, | \, X_{t}=x_i)
\]
that can be written as
\[
  b_{i j}(t) = \frac{p(X_{t-1}=x_j , \,  X_{t}=x_i)}{p(X_t = x_i)}
\]
and applying again Bayes rule we have
\[
    b_{i j}(t) =
    \frac{ p(X_{t} =x_i \, | \, X_{t-1}=x_j) p(X_{t-1} = x_j) }{ p(X_{t} = x_i) } =
    \frac{ s_{j i} p(X_{t-1} = x_j) }{ p(X_{t} = x_i)} =
    \frac{\vect{p}_{t-1}(j) s_{j i}}{\vect{p}_t(i)} = \frac{\vect{p}_{t-1}(j) s_{j i}}{\vect{p}_{t-1} \cdot S^i}
\]
If at instant $t-1$ state probability distribution is uniform then last formula becomes

\begin{equation}
\label{bij-uniforme}
b_{i j}(t) =  \frac{s_{j i}}{\sum_k s_{k i}}
\end{equation}
Note that if state probability distribution is uniform then state backward-transition matrix $B$ is a kind of transpose of the state transition matrix $S$. Note also that while $S$ is time constant, $B$ is not so, in general.

Row $B_i(t)$ of the state backward-transition matrix $B$ provides the conditional probability distribution $p(X_{t-1}\,|\,X_{t}=x_i)$ describing network state at the instant \emph{previous} to the one the network is in state $x_i$.

%%%%%%%%%%%%%%%%%%%%%%%%%%%%%%%%%%%%%%%%%%%%%%%%%%%%%%%%%%%%%%%%%%%%%%%%%%
%%%%%%%%%%%%%%%%%%%%%%%%%%%%%%%%%%%%%%%%%%%%%%%%%%%%%%%%%%%%%%%%%%%%%%%%%%
\section{Effective Information}
%%%%%%%%%%%%%%%%%%%%%%%%%%%%%%%%%%%%%%%%%%%%%%%%%%%%%%%%%%%%%%%%%%%%%%%%%%
\subsection{Introduction}
Effective information can be informally described as \emph{the quantity of information on possible predecessors of current states} acquired \emph{additionally} from actually measuring the current network state \emph{with respect to what can be acquired from the knowledge of state transition matrix only}.
We propose calling it \emph{experimental information} or \emph{Galileian information}, given the emphasis it gives to experimentally acquired knowledge with respect to purely theoretical knowledge. Here quantity of information is intended in the standard sense of the Shannon's Information Theory.

\sk
The main question effective informations answers to is: if network observation finds that its current state is $x_i$, which is the additional knowledge provided by this measure with respect to what can be known on the network by its state transition matrix only, i.e. without knowing which is the current state of the network?

Still remaining at the informal level this additional knowledge can be described as the reduction in uncertainty provided by the actual measurement with respect to the uncertainty existing on the basis of the state transition matrix only.

On one side there are those systems whose regime trajectory in the space state is a deterministic cycle. For such systems the observation provides an effective information of $\log_2 k$ bits\footnote{from now on all logarithms are to the base $2$} (where $k$ is the number of the nodes on the cycle, i.e. its length). Since a deterministic closed trajectory of length $k$ in the state space corresponds to a suitable subset of $k$ rows of the state transition matrix each containing exactly one value $1$, and since before measuring the system the uncertainty is maximum -- given that the system can be in any of these $k$ states -- while after measuring the systems it is univocally known the predecessor of the current state, the information acquired through observation is maximum and equal, according to the standard way of measuring information, to $\log k$ bits.

On the other side there are those systems whose behavior in the state space is uniformly random, that is those systems where each state can be, with equal probability, the predecessor of the current state. Measuring the actual current state in these systems provides an effective information of $0$ bits since no reduction in uncertainty is provided through the observation (complete uncertainty both before and after the measurement). Also for completely static systems, that is systems whose state is constant while time runs there is no reduction in uncertainty provided through the observation (no uncertainty either before or after the measurement).

%%%%%%%%%%%%%%%%%%%%%%%%%%%%%%%%%%%%%%%%%%%%%%%%%%%%%%%%%%%%%%%%%%%%%%%%%%
\subsection{Formal definition}
We define the effective information obtained by observing that system $X$ is in state $x_i$ at instant $t$ as
\begin{equation}
\label{ei}
  ei(t,x_i) \circeq D_{KL}(B_i(t) \; || X_{t-1})
\end{equation}
where $D_{KL}$ is the Kullback-Leibler divergence\footnote{The Kullback-Leibler divergence (or distance) of probability distribution $q(x)$ from probability distribution $p(x)$ is defined as $ D_{KL}(p||q) \circeq \sum_{x \in \Omega_x}p(x)\log{ \frac{p(x)}{q(x)}} = \avg{\log{\frac{p(x)}{q(x)}}}_p $ and note it is asymmetric.}. Then
\begin{eqnarray}
ei(t,x_i) & = &\sum_j{b_{i j}(t) \log{\frac{b_{i j}(t)}{p(X_{t-1}=x_j)}}} \nonumber \\
 & = & -H(B_i(t)) -\sum_j{b_{i j}(t)\log{p(X_{t-1}=x_j)}} \nonumber \label{ei-2}
\end{eqnarray}
Our definition is a generalization of the one provided by Tononi and coauthors (cfr. equations 1A and 1B of~\cite{BaTo-08-PCB}). Ours in fact allows to study system behavior for each time instant and for each probability distribution $X_0$, while in~\cite{BaTo-08-PCB} the time instant under investigation is always $t=1$ and it is always assumed probability distribution $X_0$ is the uniform one. Our formulation hence allows to model both the transient and the stationary regime of a system.

\sk
For the case when the state probability distribution $X_{t-1}$ is uniform the formula above becomes:
\begin{eqnarray}
ei(t,x_i) & = & -H(B_i(t)) -\sum_j{b_{i j}(t)\log{\frac{1}{2^n}}} \nonumber \\
 & = & -H(B_i(t)) +n \sum_j{b_{i j}(t)} \nonumber \\
 & = & n - H(B_i(t)) \nonumber \label{ei-unif}
\end{eqnarray}

\sk
Effective information in the regime phase of a system is provided by considering equation~\eqref{ei} in the limit for the instant $t$ tending to infinity
\[
  ei(x_i) \circeq  D_{KL}(B_i \; || X_\infty)
\]
where $B_i$ ed $X_\infty$ are the stationary probability distributions defined by the limits, if they exist, of the probability distributions for instant $t$, which describe the regime phase of the system. That is:
\[
p(X_\infty = x_i) \circeq \lim_{t \to \infty}{p(X_t = x_i)} \circeq p_i
\]
and
\begin{equation}
\nonumber \label{b-infinity}
p(B_i = x_j) \circeq p(X_{\infty} = x_j | X_{\infty} = x_i) = \frac{s_{ji}p_j}{p_i}
\end{equation}
hence
\begin{equation}
\nonumber \label{ei-infinity}
  ei(x_i)   = \sum_j{b_{i j} \log{\frac{b_{i j}}{p(X_\infty=x_j)}}}
   = -H(B_i) -\sum_j{b_{i j}\log{p_j}}
\end{equation}
A system which has a uniformly random behavior in the regime phase has $H(B_i)=n$, since state probability distribution $p(X_{t-1} | X_t)$ is $p(x_j) = \frac{1}{2^n}$, hence
\[
  ei(x_i) =  -n -\sum_j{\frac{1}{2^n}\log{\frac{1}{2^n}}} = \sum_j{\frac{n}{2^n}} -n = n-n = 0
\]
A system completely static in the regime phase, i.e. which remains fixed in a single attraction state $x_i$, has $H(B_i)=0$ since the unique possible predecessor is $x_i$ itself and $p(x_j) = 0$ if $i \neq j$ from which we have
\[
  ei(x_i) = \log{1} = 0
\]
Note that sum is computed only on observable states (i.e. where $p(x_j) \neq 0$), to avoid the undeterminate form $0\log{\frac{0}{0}}$.

A system having in the regime phase a single cyclic attractor containing all states, i.e. a deterministic closed trajectory in the space state walking through all states, has $H(B_j)=0$ since each state has exactly one predecessor while $p(x_j) = \frac{1}{2^n}$ and hence
\[
ei(x_i) = 0-\log{\frac{1}{2^n}} = n
\]
The same holds, assuming the stationary state space distribution is uniform, when the system has more cyclic attractors partitioning all the space state.

If the system has a single cyclic attractor with $k<2^n$ states (or more cyclic attractors partitioning a subset of size $k<2^n$ of all states, still assuming a uniform stationary state space distribution) then it is $ei(x_i) = \log{k}$.

\sk
The analysis in~\cite{BaTo-08-PCB} assumes the maximum uncertainty and uniformity on the initial systems conditions and is focused on computing effective information in the instant right after the initial state. The formulation of effective information in~\cite{BaTo-08-PCB} is therefore the following particular case of ours:
\[
  ei_1(x_i) = D_{KL}(B_i(1) \, || \, X_0)
\]
Note also that since for this particular case the assumptions used for the derivation of~\eqref{bij-uniforme} hold, it can be written
\[
b_{ij}(1) = \frac{s_{ji}}{\sum_k{s_{ki}}}
\]

%%%%%%%%%%%%%%%%%%%%%%%%%%%%%%%%%%%%%%%%%%%%%%%%%%%%%%%%%%%%%%%%%%%%%%%%%%
\subsection{Effective information of subsets}
\label{effective-information-subsets}
For the definition of integrated information it is required to define how to measure effective information for subsets of a given network $X$. Let $A \subseteq X$. When $X$ is in state $x_i$ we denote with $\pi_A(x_i) = {}^A\,\!x_i$ the state of $A$.
Let $A_t$ be the random variable representing state of $A$ at instant $t$. We can define for $A$ state transition matrix ${}^A\!S$ and state backward-transition matrix ${}^A\!B$ in analogy with the general case as
\[
  {}^A\!s_{ij} \circeq p(A_{t+1}=a_j \, | \, A_{t}=a_i)
\]
and
\[
  {}^A\,\!b_{i j}(t) \circeq p(A_{t-1}=a_j \, | \, A_{t}=a_i)
\]
Both can be obtained from $S$ e $p(\cdot)$ after some long but straightforward computations. Intuitively and informally speaking, the computation is based on summing transition probabilities over all states of $X$ which are equivalent with respect to subset $A$, averaged with their state probabilities.

\sk
Now, all definitions introduced for a network $X$ can be applied to any of its subset of nodes $A$ by substituting in the previous formulas $S$, $B$, and $X$ respectively with ${}^A\!S$, ${}^A\!B$, and $A$. We then obtain
\begin{equation}
\label{ei-t-A-ah}
ei(t,A,a_h) \circeq D_{KL}({}^A\!B_h(t)||A_{t-1})
\end{equation}

%%%%%%%%%%%%%%%%%%%%%%%%%%%%%%%%%%%%%%%%%%%%%%%%%%%%%%%%%%%%%%%%%%%%%%%%%%
%%%%%%%%%%%%%%%%%%%%%%%%%%%%%%%%%%%%%%%%%%%%%%%%%%%%%%%%%%%%%%%%%%%%%%%%%%
\section{Integrated Information}
%%%%%%%%%%%%%%%%%%%%%%%%%%%%%%%%%%%%%%%%%%%%%%%%%%%%%%%%%%%%%%%%%%%%%%%%%%
We are now ready to formally define integrated information, that is the quantity of information generated in a system transitioning from one state to the next by the causal interaction of its parts, above and beyond the quantity of information generated independently by each of its parts.

Given a system $X$ let $V \subseteq X$ and $\{M_k\}$ a partition of $V$ in $m$ subsets. Let $M_k(t)$ be the random variables describing the state of the $k$-th component of the partition at instant $t$.
Let $X$ be in state $x_i$ at instant $t$. Then $V$ at the same instant is in state ${}^V\!x_i$ and the $k$-th component is in state ${}^{M_k}x_i$. In the following we use $v_h$ and $\mu_k$ as a shorthand for ${}^V\!x_i$ and ${}^{M_k}x_i$, respectively.

Partition-dependent integrated information is first defined for a subset $V$ as a function of partition $\{M_k\}$, time instant $t$, and current state $v_h$ as
\begin{equation}
\label{phi-V-M}
\phi(t,V,\{M_k\},v_h) \circeq ei(t,V,v_h) - \sum_{k=1}^m{ei(t,M_k,\mu_k)}
\end{equation}
Value computed by this formula clearly depends on the considered partition. Tipically, an unbalanced partition produces a lower value of $\phi$ (see~\cite{BaTo-08-PCB}). Hence the following normalization function is introduced
\begin{equation}
\label{N} \nonumber
N(t,V,\{M_k\},v_h) \circeq (m-1)\min_k\{H(M_k(t))\}
\end{equation}
Then, the \emph{Minimum Information Partition} (\emph{MIP}) is defined as the partition providing the minimum value for the integrated information after the normalization process, that is
\begin{equation}
\label{P-MIP} \nonumber
P_{\mbox{\scriptsize \emph{MIP}}}(t,V,v_h) \circeq \argmin_P\Big \{\frac{\phi(t,V,P,v_h)}{N(t,V,P,v_h)}\Big \}
\end{equation}
The above formula has been defined by Tononi for generic partitions, but in all of its papers and here it is only discussed the case of bi-partitions, i.e. partitions in two subsets.

\sk
Integrated information $\phi$ for subset $V$, in state $v_h$ at instant $t$, is now formally defined as the value of the partition-dependent integrated information computed on \emph{MIP}, that is
\begin{equation}
\label{phi-V} \nonumber
\phi(t,V,v_h) \circeq \phi(t,V,P_{\mbox{\scriptsize \emph{MIP}}}(t,V,v_h),v_h)
\end{equation}

\sk
And it is now possible to formally define the value of integrated information for the whole system $X$.
\label{complex}
A subset $V \subseteq X$ having $\phi > 0$ is called \emph{complex}. If it is not a proper subset of another subset with a larger $\phi$ it is called \emph{main complex}. The value of integrated information of $X$, in state $x_i$ at instant $t$, is defined as the value of integrated information of its main complex of maximum value.
\begin{equation}
\label{phi} \nonumber
\phi(t,x_i) \circeq \max_{V \subset X} \, {\phi(t,V,P_{\mbox{\scriptsize \emph{MIP}}}(t,V,v_h),v_h)}
\end{equation}
The value of integrated information averaged over all states of the system is provided through the state distribution probability $p_t(\cdot)$, that is
\begin{equation}
\label{avg_phi} \nonumber
\phi(t) \circeq \sum_{x_i \in X} {\phi(t,x_i)\,p_t(i)}
\end{equation}

%%%%%%%%%%%%%%%%%%%%%%%%%%%%%%%%%%%%%%%%%%%%%%%%%%%%%%%%%%%%%%%%%%%%%%%%%%
%%%%%%%%%%%%%%%%%%%%%%%%%%%%%%%%%%%%%%%%%%%%%%%%%%%%%%%%%%%%%%%%%%%%%%%%%%
\section{Integrated information in disconnected systems}
%%%%%%%%%%%%%%%%%%%%%%%%%%%%%%%%%%%%%%%%%%%%%%%%%%%%%%%%%%%%%%%%%%%%%%%%%%
Intuitively, any system having a partition in two independent subsets, i.e. that can be partitioned in two subsets such that no node in a subset affects the state value of nodes in the other subset, should have zero as value of its integrated information.

We now give a formal proof of this property, to the best of our knowledge never appeared in the literature. We consider the value of integrated information assuming at instant $t-1$ the system has a uniform state probability distribution, consistently with discussion in~\cite{BaTo-08-PCB}. Remember that for a subset $V$ of the system $X$ in state $x_h$ we use $v_h$ as a shorthand for ${}^V\!x_h$, the restriction of $x_h$ to nodes in $V$.

\begin{theorem}[Integrated information in a disconnected network]
Let $A'$ and $A''$ be two disjoint subsets of a network $X$, $A' \cup A'' =  V \subseteq X$. Let us denote with $v_h$ the current state of $V$, and with $a'_h$ e $a''_h$ the current states of subsets $A'$ and $A''$, respectively.

For each state $v_h$ and time instant $t$ it is
\[
\phi(t,V,\{A',A''\},v_h) = 0
\]
\end{theorem}
\begin{proof}
From the definition~\eqref{phi-V-M} of partition-dependent integrated information and the definition~\eqref{ei-t-A-ah} of the effective information for a subset it is
\begin{eqnarray}
\phi(t,V,\{A',A''\},v_h) & = & ei(t,V,v_h) -ei(t,A',a'_h) -ei(t,A'',a''_h) \nonumber \\
  & = & D_{KL}({}^V\!B_h(t) \, || \, V_{t-1}) -D_{KL}({}^{A'}\!B_{i}(t) \, || \, A'_{t-1}) -D_{KL}({}^{A''}\!B_{j}(t) \, || \, A''_{t-1}) \label{all-DKL}
\end{eqnarray}
From the definition of the Kullback-Leibler divergence it is
\[
D_{KL}({}^V\!B_h(t) \, || \, V_{t-1}) = -H({}^V\!B_h(t)) -\sum_j{{}^V\!b_{h j}(t)\log{p(V_{t-1}=v_j)}}
\]
Remember that ${}^V\!B_h(t)$ is a conditional probability distribution for the state preceding the current one
\begin{eqnarray}
p({}^V\!B_h(t) = v_j) & = & p(V_{t-1} = v_j \,|\, V_{t} = v_h) \nonumber \\
 & = & p(A'_{t-1} = a'_j \wedge A''_{t-1} = a''_j \,|\, V_{t} = v_h) \nonumber
\end{eqnarray}
Applying the chain rule of entropy it is
\[
H({}^V\!B_h(t)) = H(A'_{t-1} \,|\, V_{t} = v_h) + H\Big( (A''_{t-1}
\,|\, V_{t} = v_h) \big | A'_{t-1} \Big)
\]
and given the independence between $A''$ and $A'$ it follows that
\begin{eqnarray}
H({}^V\!B_h(t)) & = & H(A'_{t-1} \,|\, V_{t} = v_h) + H(A''_{t-1} \,|\, V_{t} = v_h) \nonumber \\
 & = & H(A'_{t-1} \,|\, A'_{t} = a_{h'}) + H(A''_{t-1} \,|\, A''_{t} = a_{h''}) \nonumber \\
 & = & H({}^{A'}\!B_{h'}(t)) + H({}^{A''}\!B_{h''}(t)) \nonumber
\end{eqnarray}
From the assumption of uniform state probability distribution at $t-1$ it is
\begin{eqnarray}
D_{KL}({}^V\!B_h(t) \, || \, V_{t-1}) & = & -H({}^V\!B_h(t)) + |V| \nonumber \\
D_{KL}({}^{A'}\!B_h(t) \, || \, A'_{t-1}) & = & |A'|-H({}^{A'}\!B_h(t)) \nonumber \\
D_{KL}({}^{A''}\!B_h(t) \, || \, A''_{t-1}) & = & |A''|-H({}^{A''}\!B_h(t)) \nonumber
\end{eqnarray}
and substituting the above right members for the left ones in equation~\eqref{all-DKL} and considering that $|V| = |A'|+|A''|$ we obtain
\[
\phi(t,V,\{A',A''\},v_h) = |V| -|A'| -|A''| -H({}^V\!B_h(t)) + H({}^{A'}\!B_{h'}(t)) +
H({}^{A''}\!B_{h''}(t)) = 0
\]
%\tombstone
\end{proof}

%%%%%%%%%%%%%%%%%%%%%%%%%%%%%%%%%%%%%%%%%%%%%%%%%%%%%%%%%%%%%%%%%%%%%%%%%%
\section{Conclusions}
\label{conclusions}
In this paper we have given a thorough presentation of a model proposed by Giulio Tononi~\cite{Tono-01-AIB, ToSp-03-BMC, Tono-04-BMC} for modeling \emph{integrated information}, i.e. how much information is generated in a system by causal interaction of its parts and \emph{above and beyond} the information given by the sum of its parts. The model was aimed at trying to formally capture what is consciousness in living beings~\cite{ToEd-98-Sci, EdTo-00, Tono-08-BB} and the reader is referred to Tononi's papers for detailed motivations of the model.

\sk
We have considered the discrete version of the model~\cite{BaTo-08-PCB, Tono-08-BB, BaTo-09-PCB}. The original papers describing the model are not always fully clear in their mathematical formulation and here we have given the first formal description of such a model where all steps are detailed presented.

In doing so we have provided a more general formulation of such a model, which is independent from the time chosen for the analysis and from the uniformity of the probability distribution at the initial time instant.

Finally, we have also given here the first formal proof that a system made up by independent parts has a value of integrated information equal to zero.

\paragraph{Acknowledgments.} We would like to thank Luciano Gual\`{a} and Guido Proietti for useful and interesting discussions related to the work here described.

%%%%%%%%%%%%%%%%%%%%%%%%%%%%%%%%%%%%%%%%%%%%%%%%%%%%%%%%%%%%%%%%%%%%%%%%%
\renewcommand{\baselinestretch}{1.0}
%\normalsize
\small
%\footnotesize

\bibliographystyle{plain}
\bibliography{Integrated-Information-BIBLIO}

\begin{thebibliography}{10}

\bibitem{BaTo-08-PCB}
David Balduzzi and Giulio Tononi.
\newblock Integrated information in discrete dynamical systems: Motivation and
  theoretical framework.
\newblock {\em {PLoS} Computational Biology}, 4(6), 2008.

\bibitem{BaTo-09-PCB}
David Balduzzi and Giulio Tononi.
\newblock Qualia: The geometry of integrated information.
\newblock {\em {PLoS} Computational Biology}, 5(8), 2009.

\bibitem{EdTo-00}
Gerald~M. Edelman and Giulio Tononi.
\newblock {\em A universe of consciousness: how matter becomes imagination}.
\newblock New York, NY: Basic Books, 1 edition, 2000.

\bibitem{SDKZ-02-BI}
Ilya Shmulevich, Edward~R. Dougherty, Seungchan Kim, and Wei Zhang.
\newblock Probabilistic boolean networks: A rule-based uncertainty model for
  gene regulatory networks.
\newblock {\em Bioinformatics}, 18(2):261--274, 2002.

\bibitem{Tono-01-AIB}
Giulio Tononi.
\newblock Information measures for conscious experience.
\newblock {\em Archivi Italiani di Biologia}, 139(4):367--371, 2001.

\bibitem{Tono-04-BMC}
Giulio Tononi.
\newblock An information integration theory of consciousness.
\newblock {\em BMC Neuroscience}, 5(1):42, 2004.

\bibitem{Tono-08-BB}
Giulio Tononi.
\newblock Consciousness as integrated information: a provisional manifesto.
\newblock {\em Biol. Bull.}, December 2008.

\bibitem{ToEd-98-Sci}
Giulio Tononi and Gerald~M. Edelman.
\newblock Consciousness and complexity.
\newblock {\em Science}, 282(5395):1846--1851, 1998.

\bibitem{ToES-98-TCS}
Giulio Tononi, Gerald~M. Edelman, and Sporns Olaf.
\newblock Complexity and coherency: integrating information in the brain.
\newblock {\em Trends in Cognitive Sciences}, 2(11), 1998.

\bibitem{ToSp-03-BMC}
Giulio Tononi and Olaf Sporns.
\newblock Measuring information integration.
\newblock {\em BMC Neuroscience}, 4, December 2003.

\bibitem{ToSE-94-PNAS}
Giulio Tononi, Olaf Sporns, and Gerald~M. Edelman.
\newblock A measure for brain complexity: relating functional segregation and
  integration in the nervous system.
\newblock {\em Proceedings of the National Academy of Sciences of the United
  States of America}, 91(11):5033--5037, 1994.

\bibitem{ToSE-96-PNAS}
Giulio Tononi, Olaf Sporns, and Gerald~M. Edelman.
\newblock A complexity measure for selective matching of signals by the brain.
\newblock {\em Proceedings of the National Academy of Sciences of the United
  States of America}, 93:3422, 1996.

\end{thebibliography}

%%%%%%%%%%%%%%%%%%%%%%%%%%%%%%%%%%%%%%%%%%%%%%%%%%%%%%%%%%%%%%%%%%%%%%%%%%
\end{document}